\documentclass[sigconf,nonacm]{acmart}

\usepackage{algorithm}
\usepackage{algorithmic}
\usepackage{graphicx}
\usepackage{enumitem}
\usepackage{verbatim} 

\newtheorem{theorem}{Theorem}
\newtheorem{lemma}[theorem]{Lemma}

\newtheorem{condition}{Condition}
\newtheorem*{lemma*}{Lemma}
\newtheorem*{theorem*}{Theorem}

\DeclareMathOperator*{\argmin}{arg\,min}

\acmConference[OARS @RecSys'26]{the 20th ACM Conference on Recommender Systems}{September 27 -- October 2, 2026}{Minneapolis, MN, USA}
\acmBooktitle{Proceedings of the 20th ACM Conference on Recommender Systems (OARS @RecSys'26), September 27--October 2, 2026, Minneapolis, MN, USA}

\begin{document}

\title[BAFF: Bid-Aware Filter Family]{BAFF: Bid-Aware Filter Family for Mitigating Training Data Interference in RTB A/B Tests}

\author{Jeonglyul Oh}
\authornote{These authors contributed equally to this research.}
\email{jeonglyul@dable.io}
\affiliation{
    \institution{Dable Inc.}
    \city{Seoul}
    \country{Republic of Korea}
}

\author{Ikkyu Choi}
\authornotemark[1]
\email{ikkyu@dable.io}
\affiliation{
    \institution{Dable Inc.}
    \city{Seoul}
    \country{Republic of Korea}
}

\author{Inseop Youn}
\authornotemark[1]
\email{inseop@dable.io}
\affiliation{
    \institution{Dable Inc.}
    \city{Seoul}
    \country{Republic of Korea}
}

\author{Youngjae Kim}
\email{youngjae@dable.io}
\affiliation{
    \institution{Dable Inc.}
    \city{Seoul}
    \country{Republic of Korea}
}

\begin{abstract}
In online A/B tests for real-time bidding (RTB), control and treatment models are typically trained on a shared serving log that includes data generated by the counterpart model. This shared-log training biases each model's training data through two channels: the counterpart model may have selected a different ad from the ad-candidate pool (ad-ranking disagreement) and may have bid a different price (bid-pricing disagreement), potentially distorting the A/B test outcome. Log-splitting eliminates the bias but sacrifices training data; log-sharing retains all data but leaves the bias unaddressed. We formalize the \textbf{Bid-Aware Filter Family (BAFF)}, a class of $(k,l)$-parameterized hard filters that controls tolerance to each channel independently, providing a structured search space between these two extremes.
We further propose a three-stage online measurement protocol that enables evaluating data-sharing strategies by their deviation from an interference-free reference model in production. In offline simulation, a $(k,l)$ sweep surfaces operating points with smaller deviation from the interference-free reference model than both log-sharing and log-splitting. In a live RTB deployment on a demand-side platform (DSP), filter-based variants preserve the reference model's business metrics (e.g., CPC, CTR) more closely than both baselines. The best operating point is setting-dependent, underscoring the practical value of the search space itself.

\end{abstract}



\maketitle

\section{Introduction}
Standard practice in online A/B testing is to train both control and treatment models on the full pooled serving log, including entries generated by the counterpart model~\cite{brennan2025symbiosis}. However, sharing logs biases each model's training dataset when the two models have discrepancies in their predictions. Brennan et al.~\cite{brennan2025symbiosis} call this \emph{symbiosis bias} and frame it as a SUTVA violation caused by sharing training data. Si~\cite{si2023interference} studies a related phenomenon under the name \emph{interference induced by data training loops}, and Zheng and Zhao~\cite{zheng2025adaptation} report \emph{algorithm adaptation} effects in production recommenders. These previous works establish the bias as a general property of shared-log training and propose experimental design remedies. Left unaddressed, the bias can distort an A/B test's winner--loser comparison, causing operators to discard a superior model or deploy an inferior one.

We focus on RTB, where a demand-side platform (DSP) receives a bid request, selects an ad from its candidate pool, and submits a bid---both decisions determined by the model's score estimates (e.g.\ pCTR). Since only the top-ranked ad is served and only the winning bid produces a log entry, any discrepancy between the two models directly biases the treatment model $B$'s training data. $B$ trains on impressions that the control model $A$ selected and won---ads that $B$ may not have ranked first, in auctions that $B$'s bid may not have won---shifting $B$'s learned distribution away from what it would have generated under its own serving.
A natural and straightforward response is to stop sharing logs: train each model only on its own logs. 
In production RTB at deployment scale, reducing the size of the training dataset measurably degrades model quality---a cost that Brennan et al.~\cite{brennan2025symbiosis} flag for the data-diverted design and this is independently supported by empirical data-scaling laws for recommendation models~\cite{ardalani2022}.

The trade-off comes down to how much of the logs from the counterpart model to retain in training. Log-sharing retains all of them (data-rich but biased) and Log-splitting retains none (unbiased but data-starved) of them. 
Between these extremes lies a continuum of strategies that decides whether to keep using each row or drop some data based on which the two models would have made substantially different ranking or bidding decisions.
Any such strategy is defined by its tolerance to each kind of disagreement --- ad ranking and bid pricing --- with log-sharing as the maximally permissive endpoint and log-splitting as the maximally strict one. We formalize this middle path as the \textbf{Bid-Aware Filter Family (BAFF)}, a class of hard filters that keeps the shared log but drops some logs on which both models vary enough to distort either the top-ranked ad or the submitted bid-price.

Our contributions are as follows:
\begin{itemize}[nosep]
\item We decompose training-data interference in RTB into two observable channels: ad-ranking and bid-pricing disagreement. Building on this decomposition, we formalize the \textbf{Bid-Aware Filter Family}, a $(k,l)$-parameterized filter family that interpolates between log-sharing and log-splitting, providing a structured search space for locating the least-biased data-sharing strategy in a given deployment.
  \item We propose a three-stage online measurement protocol that enables direct comparison of data-sharing strategies for identifying the least-biased option in production.
  \item We validate the framework in both simulation and a live RTB deployment, observing $(k,l)$ operating points that are less biased than log-sharing and log-splitting.
\end{itemize}

The remainder of the paper is organized as follows. Section~\ref{sec:rel_work} surveys related work. Section~\ref{sec:method} defines the BAFF and analyzes a ridge-regression surrogate that motivates the filter's design. Section~\ref{sec:simulation} validates the framework in a controlled simulation with known ground truth. Section~\ref{sec:online} presents a three-stage measurement protocol and reports results from a live RTB deployment. Section~\ref{sec:limitations} discusses limitations and future work, and Section~\ref{sec:conclusion} concludes.

\section{Related Work}
\label{sec:rel_work}

\paragraph{Bias in ML A/B tests.}
Prior work on training data interference in ML A/B tests has developed almost entirely outside real-time bidding (RTB), under settings whose log-generation mechanism differs from ours. In two-sided recommender and marketplace platforms, \citet{jeunen2023} first articulated that pooled training data couples the two arms of an A/B test through model interference even absent user-level network effects, and \citet{brennan2025symbiosis} formalized this coupling as \emph{symbiosis bias} via a theoretical model, comparing cluster-randomized, data-diverted, and user-corpus co-diverted designs through simulation and validating symbiosis bias empirically on data from a large-scale global-recommender A/B test; the country-diverted design used as a practical mitigation of network effects on a production recommender was deployed by \citet{lin2024}. A related but distinct line on marketplace bipartite interference \citep{brennan2022bipartite} targets two-sided platforms where units interact through the other side of the graph rather than through a shared training set. A second line treats interference as a multi-armed-bandit phenomenon and asks when shared data still lets A/B experiments rank algorithms correctly \citep{li2025choosing}. A third line intervenes at the market mechanism rather than the training set---budget pacing \citep{liao2024}, ranked-list position competition \citep{goli2024bias}, divergent creative delivery \citep{braun2024divergent}, and network-aware randomization \citep{jiang2024mixed}---and the AI-feedback-loop survey of \citet{stoecker2025bias} catalogs $24$ such studies in recommender systems. None of these share RTB's defining property that a training row exists only when the bidding policy wins the auction, so which arm enters the pooled log is endogenously determined by both ad selection \emph{and} bid price. Our work addresses this RTB-specific coupling directly.

\paragraph{Off-policy evaluation and learning.}
A complementary line replaces online A/B with logged-bandit analysis, but its RTB instantiations remain focused on single-policy evaluation and do not address training data interference. Off-policy evaluation (OPE) of a frozen target policy is well developed in advertising, from the counterfactual reasoning framework of \citet{bottou2013counterfactual} to IPS-based empirical risk for logged bandit feedback \citep{swaminathan2015counterfactual}, counterfactual estimators benchmarked against online A/B outcomes on a production recommender \citep{gilotte2018offline}, and deficient-support diagnoses that propose remedies including restricting the policy class to the logging support \citep{sachdeva2020deficient}; RTB-native estimators further address auction-specific degeneracies---\citet{yeom2024dpmope} repurposed the bid-landscape model as a propensity approximator for OPE under deterministic winner-takes-all selection, validated against a 2-week online A/B test---while \citet{barajas2021ghost} use double-blind ghost bidding to measure ad-exposure incrementality in a running programmatic campaign, and \citet{jeunen2023auctiongym} benchmark value-, policy-, and doubly-robust off-policy learning approaches to bidding in simulation. These address single-policy evaluation or frozen-policy incrementality, not the training-time interference that arises when two learning policies share a pooled RTB log. On the off-policy learning (OPL) side, the closest work is \citet{si2023interference}, who trains each arm with a weighted loss whose weights come from a classifier predicting whether a data point was generated under treatment or control; their setting, however, is a generic recommender A/B without bidding, budget, or revenue-linked outcomes, and the approach is evaluated only in synthetic simulation. Direct transfer to RTB is non-trivial because weights must account for the stochastic bid landscape that determines whether a bid wins, a quantity absent from non-auction settings. To the best of our knowledge, (i) no prior work formulates OPL for training data interference in RTB, and (ii) no prior work demonstrates online mitigation of this interference on an RTB system. Our Bid-Aware Filter Family and three-stage measurement protocol close both gaps.

\section{Methodology}
\label{sec:method}
In RTB, one observable form of the bias induced by sharing logs is the disagreement between the two models' outputs. On each incoming bid request $x$, a model scores every candidate ad by combining its estimated score with the advertiser's valuation, selects the ad maximizing the resulting score, and submits that maximum as the bid price. Both the selected ad and the submitted bid are deterministic functions of the model's scores. Any disagreement between models $A$ and $B$ therefore propagates to the logged request through two observable channels, the selected ad and the submitted bid. These are the channels along which differences in the control model's logs enter the treatment model's training set.

The remainder of this section is organized as follows. Section~\ref{subsec:filterdefinition} defines the filter family. Its two axes are the two observable channels identified above, each capturing disagreement along the corresponding channel. Section~\ref{subsec:algorithm} presents the algorithm that applies this filter as an offline preprocessing step, before the treatment model trains. Section~\ref{subsec:theory} then analyzes a ridge-regression surrogate whose closed-form bound decomposes the filtered estimator's parameter deviation into three factors. This decomposition provides the theoretical motivation for the filter's design.

\subsection{Filtering Criterion}
\label{subsec:filterdefinition}

\paragraph{Setup.}
We operate as a demand-side bidder within an RTB. For each incoming bid request $x \in \mathcal{X}$, an internal traffic router assigns $x$ to one of our models. The assigned model ranks the $K$ candidate ads available for the request, selects a winner $a \in \{1, 2, \dots, K\}$, computes a bid price from its predicted value for $(x, a)$, and submits the bid to the ad exchange, where the submitted bid competes against bids from other demand-side participants. A training example is produced only when our bid wins the auction. The chosen ad is then served as an impression, and the user's feedback, such as a click or conversion, is attached to $(x, a)$.

\paragraph{Disagreement scores.}
Let $a_A(x), a_B(x) \in \{1, 2, \dots, K\}$ be the ad that each model would have selected on request $x$ and $bp_A(x), bp_B(x)\allowbreak > 0$ be the corresponding bid prices. Since both models can score the full candidate pool, these counterfactual selections and bid prices are restorable offline for any logged request. For a request $x$ logged by the model $A$, we quantify disagreement scores on two axes,
\[
d_{ad}(x):=\frac{\text{rank}_B(a_A(x) | x)}{K}, d_{bp}(x) := \frac{(bp_A(x)-bp_B(x))^+}{bp_A(x)},
\]
where $\text{rank}_B(\cdot|x) \in \{0, 1, \dots, K-1\}$ is the 0-based rank induced by the model $B$'s score (rank 0 $= B$'s top selection), and $(\cdot)^+=\max\{0, \cdot\}$. Both disagreement scores lie in $[0, 1)$ and vanish exactly when the two models agree on the corresponding axis.

\paragraph{Bid-Aware Filter Family.}
Given tolerances $(k, l) \in [0, 1]^2$, the \emph{Bid-Aware Filter} $\mathcal{F}_{k, l}$ retains a log produced by the control model $A$ for inclusion in the treatment model $B$'s training set if and only if
\[
d_{ad}(x) < k \;\land\; d_{bp}(x) < l.
\]
The parameterized collection $\{\mathcal{F}_{k, l} : (k, l) \in [0, 1]^2\}$ is the \emph{Bid-Aware Filter Family}. Letting $\varepsilon > 0$ denote an arbitrarily small tolerance, Table~\ref{tab:corners} lists the retained control logs at each corner of the parameter square. Two regimes correspond to standard baselines. On the boundary $\{k = 0\} \cup \{l = 0\}$, no control log passes the filter and each model trains on its own logs alone, a regime we call \emph{Log-Split}. At the opposite corner $(1, 1)$, every control log passes and both models share the full pooled log without filtering, a regime we call \emph{Naive}. The three remaining corners apply strict filtering on both axes or on a single axis. Interior values interpolate between these extremes. Here $k$ bounds the permitted ad-rank disagreement as a fraction of the candidate size, and $l$ bounds the permitted bid-price disagreement as a fraction of $bp_A$.

\begin{table}[ht]
  \centering
  \small
  \caption{Filter behavior at corners of the parameter square.}
  \label{tab:corners}
  \begin{tabular}{@{}lll@{}}
    \toprule
    $(k, l)$ & retained control logs & name \\
    \midrule
    $k = 0 \lor l = 0$ & none & Log-Split \\
    $(\varepsilon, \varepsilon)$ & $B$ selects $a_A$ and $bp_A \le bp_B$ & \\
    $(\varepsilon, 1)$ & $B$ selects $a_A$ & \\
    $(1, \varepsilon)$ & $bp_A \le bp_B$ & \\
    $(1, 1)$ & all control logs & Naive \\
    \bottomrule
  \end{tabular}
\end{table}

\subsection{Algorithm}
\label{subsec:algorithm}

The Bid-Aware filter $\mathcal{F}_{k, l}$ is realized as an offline step in the training-data assembly pipeline, executed whenever model $B$ retrains on its sliding window of recent logs.

\begin{algorithm}[ht]
  \caption{Bid-Aware Filtering}
  \label{alg:bid-aware}
  \begin{algorithmic}[1]
    \REQUIRE Control logs $\mathcal{D}_A$; treatment logs $\mathcal{D}_B$; tolerances $(k, l) \in [0, 1]^2$
    \STATE $\mathcal{D}_{\mathrm{train}} \gets \mathcal{D}_B$
    \FORALL{$(x, a_A, bp_A, y_A) \in \mathcal{D}_A$}
      \STATE Replay request $x$ through model $B$ to obtain its selected ad $a_B$ and bid price $bp_B$
      \STATE $d_{\mathrm{ad}} \gets \mathrm{rank}_B(a_A \mid x) \,/\, K$
      \STATE $d_{\mathrm{bp}} \gets \max(bp_A - bp_B,\; 0) \,/\, bp_A$
      \IF{$d_{\mathrm{ad}} < k$ \AND $d_{\mathrm{bp}} < l$}
        \STATE $\mathcal{D}_{\mathrm{train}} \gets \mathcal{D}_{\mathrm{train}} \cup \{(x, a_A, y_A)\}$
      \ENDIF
    \ENDFOR
    \STATE Train $B$ on $\mathcal{D}_{\mathrm{train}}$
  \end{algorithmic}
\end{algorithm}

Note that at serving time, the filter is inactive. The model continues to score, select, and bid as usual. Only at the training time, Algorithm~\ref{alg:bid-aware} adds a single forward pass of the model $B$ per log from the control model.

\subsection{Theoretical Analysis}
\label{subsec:theory}
This section asks how the filter's choice shapes the learned parameters of the treatment model. MLP-based architectures are non-convex, which rules out a closed-form description of how training-pool composition shifts the optimum. We therefore analyze a ridge-regression surrogate. In both ridge and the MLP, a single parameter vector is updated by every training sample, and this is the mechanism through which training-pool composition moves the optimum. The advantage of ridge is that its optimum can be written down in closed form. The main result (Theorem~\ref{thm:main}) bounds the parameter deviation of the filtered estimator by a product of three factors, two controlled by the filter and one by the feature geometry alone. The decomposition sets up the filter-design discussion that follows.

\paragraph{Setup.}
Consider an A/B test with the control model $A$ and the treatment model $B$. We analyze log distributions over features and labels $(z, y) \in \mathbb{R}^d \times \mathbb{R}$. Let $\mathcal{D}_B$ denote the log distribution of the treatment model $B$, and let $\mathcal{D}_F$ (resp. $\mathcal{D}_R$) be the retained (resp. removed) components of the control model $A$'s log distribution under the filter. Let $\alpha_B, \alpha_F, \alpha_R > 0$ with $\alpha_B + \alpha_F + \alpha_R = 1$ be their mixture proportions and let
\[
\mathcal{D}_B,\qquad
\mathcal{D}_{B+F} := \tfrac{\alpha_B \mathcal{D}_B + \alpha_F \mathcal{D}_F}{\alpha_B + \alpha_F},\qquad
\mathcal{D}_N := \alpha_B \mathcal{D}_B + \alpha_F \mathcal{D}_F + \alpha_R \mathcal{D}_R,
\]
be the $B$-only, filtered, and naive training pools, respectively. For each $P \in \{B, B{+}F, N\}$, let
\[
\theta^*_P := \argmin_\theta\, \mathbb{E}_P[(y - \theta^\top z)^2] + \lambda \|\theta\|^2
\]
denote the ridge optimum on pool $P$ at penalty $\lambda > 0$.
For each pool $P$, let $\Sigma_P := \mathbb{E}_P[zz^\top]$. Define the \emph{residual}
\[
\varepsilon^* \;:=\; y - \theta^{*\top}_B z,
\]
and, for any pool $S$, the \emph{mismatch}
\[
\delta^*_S \;:=\; \mathbb{E}_S[z\varepsilon^*] - \mathbb{E}_B[z\varepsilon^*].
\]
The mismatch $\delta^*_S$ records how much the feature-residual cross-moment $\mathbb{E}[z\varepsilon^*]$ shifts between $\mathcal{D}_B$ and pool $S$, that is, nonzero $\delta^*_S$ signals that pool $S$'s data systematically pulls the ridge optimum away from $\theta^*_B$.
Throughout this section, we use $\|\cdot\|$ for both the Euclidean norm on $\mathbb{R}^d$ and the spectral (operator) norm $\sup_{\|v\|=1}\|Av\|$ on matrices, which satisfies $\|Av\| \leq \|A\|\,\|v\|$ by definition.

\begin{lemma}[Core identity]\label{lem:core}
For each $P \in \{B,\, B{+}F,\, N\}$,
\[
\theta^*_P - \theta^*_B \;=\; (\Sigma_P + \lambda I)^{-1}\, \delta^*_P.
\]
\end{lemma}
\begin{proof}
Follows from the first-order condition for $\theta^*_P$ together with $y = \theta^{*\top}_B z + \varepsilon^*$. The rest of the proof is omitted.
\end{proof}

Taking norms and applying submultiplicativity gives
\[
\|\theta^*_P - \theta^*_B\| \;\leq\; \|(\Sigma_P + \lambda I)^{-1}\| \cdot \|\delta^*_P\|,
\]
thus it suffices to bound the two factors on the right.

\begin{lemma}[Mismatch decomposition]\label{lem:decomp}
\[
\delta^*_N \;=\; \alpha_F\, \delta^*_F \;+\; \alpha_R\, \delta^*_R,
\qquad
\delta^*_{B+F} \;=\; \frac{\alpha_F}{\alpha_B + \alpha_F}\, \delta^*_F.
\]
\end{lemma}
\begin{proof}
Follows from linearity of expectation under the mixture decompositions of $\mathcal{D}_N$ and $\mathcal{D}_{B+F}$. The rest of the proof is omitted.
\end{proof}

The two pools thus differ qualitatively: $\delta^*_{B+F}$ inherits only the $F$-component mismatch, whereas $\delta^*_N$ carries both $F$ and $R$ contributions. Whether filtering yields a smaller parameter deviation depends on the relative magnitudes of $\delta^*_F$ and $\delta^*_R$.

\begin{condition}[Mismatch dominance]\label{cond:dom}
There exists $\beta > \alpha_F/\alpha_R$ such that
\[
\|\delta^*_R\| \;\geq\; \beta\, \|\delta^*_F\|.
\]
\end{condition}

\begin{lemma}[Mismatch lower bound on Naive]\label{lem:lb}
Under Condition~\ref{cond:dom},
\[
\|\delta^*_N\| \;\geq\; (\alpha_R \beta - \alpha_F)\, \|\delta^*_F\|.
\]
\end{lemma}
\begin{proof}
The proof is given in Appendix~\ref{app:proof-lb}.
\end{proof}

\begin{theorem}[Main bound]\label{thm:main}
Under Condition~\ref{cond:dom},
\[
\|\theta^*_{B+F} - \theta^*_B\|
\;\leq\;
\underbrace{\frac{\alpha_F}{\alpha_B + \alpha_F}}_{A}
\cdot
\underbrace{\frac{1}{\alpha_R \beta - \alpha_F}}_{B}
\cdot
\underbrace{\left(1 + \frac{\|\Sigma_N\|}{\lambda}\right)}_{C}
\cdot
\|\theta^*_N - \theta^*_B\|.
\]
\end{theorem}
\begin{proof}
Combine Lemma~\ref{lem:core} (applied to both $P = B{+}F$ and $P = N$), Lemma~\ref{lem:decomp}, Lemma~\ref{lem:lb}, and submultiplicativity of the spectral norm with $\|(\Sigma_{B+F} + \lambda I)^{-1}\| \leq 1/\lambda$. The detailed proof is given in Appendix~\ref{app:proof-main}.
\end{proof}

\paragraph{Anatomy of the bound.}
Theorem~\ref{thm:main} decomposes the parameter deviation of the filtered estimator into three factors. Factor $A = \alpha_F/(\alpha_B + \alpha_F)$ is the share of the filtered pool drawn from the control model, and quantifies how much control exposure the filter retains. Factor $B = 1/(\alpha_R\beta - \alpha_F)$ shrinks when mismatch on the removed set dominates mismatch on the retained set, and measures how cleanly the filter separates the two. Factor $C = 1 + \|\Sigma_N\|/\lambda$ is a purely geometric quantity, determined by the feature distribution and $\lambda$ alone. The decomposition isolates \emph{how much} is shared, \emph{which} portion of what is shared is harmful, and \emph{how} the feature geometry amplifies what remains.

\paragraph{Selectivity as the design objective.}
Factor $A$ can be shrunk mechanically by rejecting more logs, but in the limit $\alpha_F \to 0$ no control log is retained at all, and the filtered estimator reduces to $\theta^*_B$. The distinguishing work of a filter is therefore done along factor $B$, which concerns not \emph{how many} control logs are retained but \emph{which} ones. A good filter concentrates the systematic disagreement with the treatment model in the logs it removes. Theorem~\ref{thm:main} should thus be read as a statement about \emph{what a filter ought to separate}, not how aggressive it should be.

\paragraph{Implications for BAFF}
The filter family $\{\mathcal{F}_{k, l}\}$ introduced in Section~\ref{subsec:filterdefinition} can be read through this lens. Its two axes, ad-rank disagreement $d_{ad}$ and bid-price disagreement $d_{bp}$, are observable proxies for the mismatch direction $\mathbb{E}[z\varepsilon^*]$ that factor $B$ asks the filter to separate. By rejecting logs on which the two models' outputs diverge along either axis, the filter concentrates systematic disagreement in the removed set. Theorem~\ref{thm:main} does not prescribe a particular $(k, l)$, but it provides the framework within which the filter family is a principled approximation to the selectivity objective.

\paragraph{Scope.}
Theorem~\ref{thm:main} is a population statement about ridge optima, and Condition~\ref{cond:dom} is an assumption on the filter rather than a property we derive. The closed-form identity does not carry over literally to non-convex, finite-sample training, but the underlying mechanism transfers. Pooled training pulls the optimum toward the population mismatch direction, and the three factors identify what controls the strength of that pull. Section~\ref{sec:online} verifies this empirically for deep CTR models.

            \section{Simulation Experiment}
            \label{sec:simulation}

    This section validates the BAFF in a simulated RTB auction where the ground truth is known by construction. We sweep the $(k, l)$ grid and measure how closely each filtered model recovers the policy of an interference-free reference.

            \subsection{Environment}

        Table~\ref{tab:symbols} summarizes the notation used throughout this section.

        \begin{table}[ht]
        \caption{Key symbols used in the simulation setup.}
        \label{tab:symbols}
        \centering
        \small
        \begin{tabular}{@{}lll@{}}
        \toprule
        Symbol & Description & Value / Distribution \\
        \midrule
        $x$                           & User context                              & $\mathcal{N}(0, I_{100})$ \\
        $K$                           & Ad candidate pool size                    & 30 \\
        $v_a$                         & Ad value                                  & $\text{LogNormal}(0.1, 0.2)$ \\
        $A, B$                        & Control / treatment models                & --- \\
        $B_{\text{init}}$             & $B$ after Phase~0 training (frozen)       & --- \\
        $\hat{p}_M(x, a)$             & Model $M$'s predicted CTR                 & --- \\
        $\mathcal{D}_B^{\text{full}}$ & Universe-$B$ log (reference)            & --- \\
        $\mathcal{D}_A, \mathcal{D}_B$ & Universe-$AB$ logs (50/50 split)        & --- \\
        $\pi_M(a \mid x)$             & Evaluation policy for model $M$           & $\propto \exp(\hat{p}_M(x,a)\, v_a)$ \\
        \bottomrule
        \end{tabular}
        \end{table}

        We build on the CTR model and ad candidate structure of AuctionGym~\cite{jeunen2023auctiongym}, but replace its multi-agent auction with a single-agent setup suited to our A/B testing scenario. A user arrives with context $x \sim \mathcal{N}(0, I_{100})$, and a model must select one of $K = 30$ ads, each characterized by a fixed embedding $\boldsymbol{\phi}_a$, bias $\beta_a$, and value $v_a \sim \text{LogNormal}(0.1, 0.2)$, all drawn once and held constant throughout the simulation. The probability that the user clicks on a shown ad is governed by a true CTR:
        \[
            \text{CTR}(x, a) = \sigma(x^\top \boldsymbol{\phi}_a + \beta_a),
        \]
        where $\sigma(\cdot)$ is the sigmoid function. Ad selection and bidding follow Section~\ref{sec:method}: each model $M$ selects $a_M(x)$ and bids $bp_M(x)$. Following a first-price auction, an impression is won iff $bp_M(x) \geq \text{mp}(x)$, and the winner pays its own bid. Rather than simulating competing bidders, we model competition via a synthetic market price derived from the \emph{true} CTR:
        \[
            \text{mp}(x) = \max_{a} \bigl(\text{CTR}(x,a)\, v_a\bigr) \cdot \gamma \cdot \eta, \quad \eta \sim \text{LogNormal}(0, 0.3),
        \]
        where $\gamma \in (0, 1)$ controls the competitiveness of the market.
        The term $\max_a(\text{CTR}(x,a)\, v_a)$ is the highest achievable expected value for request $x$, so $\gamma = 0.5$ places the market price at roughly half of this ceiling, representing a moderately competitive market.
        The multiplicative noise $\eta \sim \text{LogNormal}(0, 0.3)$ adds per-request price variation (90\% of draws within $\times 0.6$--$1.6$).
        This provides a controllable auction threshold without requiring a full multi-agent simulation. When our bid wins, a click is drawn as $y \sim \text{Bernoulli}\bigl(\text{CTR}(x, a_M(x))\bigr)$.

        \paragraph{Model configuration.} We use $x \in \mathbb{R}^{100}$ throughout; the true CTR and market price operate on the full context. Both $A$ and $B$ are architecturally restricted to a prefix of $x$: $A$ observes $x_{1:20}$ and $B$ observes $x_{1:50}$, with the remaining dimensions acting as residual confounders. Architecture, hyperparameters, and all numeric settings are reported in Appendix~\ref{app:impl}.

        \subsection{Multiverse Design}
        \label{subsec:multiverse-design}

        We use a \emph{multiverse} design that provides an interference-free ground truth.

        \paragraph{Phase 0 (Initialization).} A pool of 100{,}000 synthetic warm-up users is drawn from the same distribution. To produce $A$'s serving log we face a circular dependency: training $A$ requires a serving log, but generating a serving log requires a trained $A$. We resolve this with a \emph{surrogate model for $A$}: instead of learning an MLP, we compute pCTR directly from the true data-generating process (DGP) parameters restricted to $A$'s 20 observable features and add Gaussian noise (std $= 0.2$) to approximate the prediction error of a converged model. This surrogate replaces only the pCTR used for ad selection and bidding; user contexts $x$, market prices, and click labels are all generated from the full DGP (100 dimensions) as usual. The resulting impression log is used to fit both $A$ and $B_{\text{init}}$ on the \emph{same} data, so downstream disagreement during Phases~1--2 arises from their differing observable-feature sets, not from different training histories.

        \paragraph{Phase 1 (Multiverse split).} We run two parallel universes from the same user population:
        \begin{itemize}[nosep]
            \item \textbf{Universe-$B$}: $B_{\text{init}}$ serves 100\% $\rightarrow$ $\mathcal{D}_B^{\text{full}}$ (ground truth)
            \item \textbf{Universe-$AB$}: $A$ and $B_{\text{init}}$ at 50/50 $\rightarrow$ $\mathcal{D}_A, \mathcal{D}_B$
        \end{itemize}
        Here $\mathcal{D}_A$ and $\mathcal{D}_B$ denote the serving logs of $A$ and $B_{\text{init}}$, respectively. Models are frozen during collection. $\mathcal{D}_B^{\text{full}}$ has roughly twice the volume of $\mathcal{D}_B$, since Universe-$B$ serves $B_{\text{init}}$ to 100\% of traffic; filtered $\mathcal{D}_A$ subsets unioned with $\mathcal{D}_B$ can partially offset this gap. Duration and traffic volume are reported in Appendix~\ref{app:impl}.

        \paragraph{Phase 2 (Filtering, Training \& Evaluation).} We run the scoring and bidding rule of the frozen $B_{\text{init}}$ on every stored context $x \in \mathcal{D}_A$ to obtain $(a_B(x), bp_B(x))$, then apply the $(k, l)$ filter (Section~\ref{subsec:filterdefinition}). Re-scoring reveals that $80.1\%$ of $\mathcal{D}_A$ impressions have different ad selections between $A$ and $B_{\text{init}}$, whereas only $7.3\%$ have lower bid prices under $B_{\text{init}}$ than under $A$. Each filtered subset is unioned with $\mathcal{D}_B$ to train a candidate model; a reference model is trained on $\mathcal{D}_B^{\text{full}}$.

        To evaluate, we draw held-out users from the same context distribution (details in Appendix~\ref{app:impl}) and compute, for each trained model, the evaluation policy
        \[
            \pi_M(a \mid x) = \frac{\exp\bigl(\hat{p}_M(x,a) \cdot v_a\bigr)}{\sum_{a'=1}^{K} \exp\bigl(\hat{p}_M(x,a') \cdot v_{a'}\bigr)},
        \]
        where $M$ is the model under evaluation. This softmax converts each model's value-weighted scores into a stochastic policy. We report $\text{KL}\bigl(\pi_{\text{ref}} \,\|\, \pi_{\text{cand}}\bigr)$, where $\pi_{\text{ref}}$ is the policy of the reference model trained on $\mathcal{D}_B^{\text{full}}$ and $\pi_{\text{cand}}$ is the policy of each filtered candidate model; lower KL indicates less policy distortion from the interference-free reference. All results are averaged over 10 independent seeds; statistical details are in Appendix~\ref{app:impl}.

            \subsection{Results}
            \label{sec:results}

            We sweep the Bid-Aware filter $\mathcal{F}_{k,l}$ (Section~\ref{subsec:filterdefinition}) over a $3 \times 3$ grid and report KL divergence to the Universe-$B$ reference policy. Note that $k{=}\varepsilon$ corresponds to the strict ad filter: $d_{ad}(x) < \varepsilon$ holds iff $a_A(x)$ is $B$'s selected ad, recovering the $(\varepsilon, 1)$ corner of the filter family. Table~\ref{tab:grid} shows the KL magnitude for each $(k, l)$ cell. Tables~\ref{tab:wins} and~\ref{tab:wins-logsplit} test whether non-trivial $(k, l)$ configurations improve over the two corner baselines---Naive $(k{=}1,\, l{=}1)$ and Log-Split $(\{k{=}0\} \cup \{l{=}0\})$---via paired one-sided $t$-tests.

            \begin{table}[ht]
            \caption{Policy distortion (KL $\times 10^{-3}$, mean $\pm$ std, 10 seeds) across the $(k, l)$ grid. Parenthesized values show the fraction of $\mathcal{D}_A$ retained by each filter. Rows $l{=}1$ and $l{=}0.5$ yield identical filtered subsets because bid disagreement is sparse (cf.\ Phase~2).}
            \label{tab:grid}
            \centering
            \small
            \begin{tabular}{lccc}
            \toprule
            & $k{=}\varepsilon$ & $k{=}0.5$ & $k{=}1$ \\
            \midrule
            $l{=}\varepsilon$  & 0.54$\pm$0.48\;(20\%) & 3.29$\pm$3.12\;(60\%) & 6.82$\pm$3.54\;(92\%) \\
            $l{=}0.5$          & 0.51$\pm$0.46\;(22\%) & 2.87$\pm$3.08\;(65\%) & 6.61$\pm$2.59\;(100\%) \\
            $l{=}1$            & 0.51$\pm$0.46\;(22\%) & 2.87$\pm$3.08\;(65\%) & 6.61$\pm$2.59\;(100\%) \\
            \midrule
            \multicolumn{4}{l}{Log-Split ($k{=}0$ or $l{=}0$): \quad 0.80$\pm$0.55\;(0\%)} \\
            \bottomrule
            \end{tabular}
            \end{table}

            \begin{table}[ht]
            \caption{Wins vs Naive (paired one-sided $t$-test across 10 seeds). Each cell shows wins/$N$ for the alternative hypothesis that the $(k, l)$ filter achieves lower KL than Naive. $^*$Significant at $p < 0.05$. $^\dagger$Self-comparison (Naive).}
            \label{tab:wins}
            \centering
            \small
            \begin{tabular}{lccc}
            \toprule
            & $k{=}\varepsilon$ & $k{=}0.5$ & $k{=}1$ \\
            \midrule
            $l{=}\varepsilon$  & \textbf{10/10}$^*$ & \textbf{9/10}$^*$ & 6/10 \\
            $l{=}0.5$          & \textbf{10/10}$^*$ & \textbf{7/10}$^*$ & 0/10 \\
            $l{=}1$            & \textbf{10/10}$^*$ & \textbf{7/10}$^*$ & ---$^\dagger$ \\
            \bottomrule
            \end{tabular}
        \end{table}

            \begin{table}[ht]
            \caption{Wins vs Log-Split (paired one-sided $t$-test across 10 seeds). Each cell shows wins/$N$ for the alternative hypothesis that the $(k, l)$ filter achieves lower KL than Log-Split. $^*$Significant at $p < 0.05$.}
            \label{tab:wins-logsplit}
            \centering
            \small
            \begin{tabular}{lccc}
            \toprule
            & $k{=}\varepsilon$ & $k{=}0.5$ & $k{=}1$ \\
            \midrule
            $l{=}\varepsilon$  & \textbf{9/10}$^*$  & 3/10 & 0/10 \\
            $l{=}0.5$          & \textbf{10/10}$^*$ & 3/10 & 0/10 \\
            $l{=}1$            & \textbf{10/10}$^*$ & 3/10 & 0/10 \\
            \bottomrule
            \end{tabular}
        \end{table}

            \paragraph{Beyond the corner baselines.}Table~\ref{tab:grid} shows that Log-Split already achieves an $8.2\times$ KL reduction over Naive (KL $\times 10^{-3}$: 0.80 vs 6.61), consistent with the expectation that unfiltered $\mathcal{D}_A$ degrades the trained policy (cf.\ Section~\ref{sec:method}). More broadly, seven of the eight non-Naive cells in Table~\ref{tab:wins} achieve lower KL than Naive in a majority of paired seeds. However, Table~\ref{tab:wins-logsplit} reveals that the search space contains strategies that go further: $k{=}\varepsilon$ cells outperform Log-Split (all $p \leq 0.008$), demonstrating that even retaining only ${\sim}20\%$ of $\mathcal{D}_A$ through well-targeted filtering provides information beyond $\mathcal{D}_B$ alone. This motivates the $(k, l)$ framework as a practical tool for locating operating points that neither obvious baseline can reach.

            \paragraph{Ad axis dominance in this DGP} Table~\ref{tab:grid} shows that moving from $k{=}1$ to $k{=}\varepsilon$ reduces KL by more than an order of magnitude (6.61 $\to$ 0.51), making the $k{=}\varepsilon$ column the best-performing region of the grid. Varying $l$ has negligible effect. This reflects the mismatch structure of this environment, where ad-ranking disagreement is far more prevalent than bid-pricing disagreement (80.1\% vs.\ 7.3\% of $\mathcal{D}_A$; cf.\ Phase~2).

            \paragraph{Setting-dependence of the optimal $(k, l)$.} The relative importance of each axis depends on the mismatch between $A$ and $B$. In environments with higher bid disagreement (e.g., different bid shading strategies), the $l$ axis may contribute independently. The $(k, l)$ grid provides a practical search space for locating the best operating point without committing to a fixed strategy.

\section{Online Experiment}
\label{sec:online}

Building on the Bid-Aware Filter Family introduced in Section~\ref{sec:method}, we instantiate the family at four operating points of $[0,1]^2$ and ask whether the filters preserve the reference model's operating point on the most business-relevant online metric, CPC (cost per click, the average advertiser spend per click). We describe the production setting (Section~\ref{sec:online-prod-setting}), a three-stage measurement protocol (Section~\ref{sec:online-protocol}), the evaluation criterion (Section~\ref{sec:reference-preservation}), and the results of a case study on a single advertiser--SSP (Supply-side Platform) pair (Section~\ref{sec:online-results}).

            \begin{figure*}[t]
              \centering
              \includegraphics[width=0.76\textwidth]{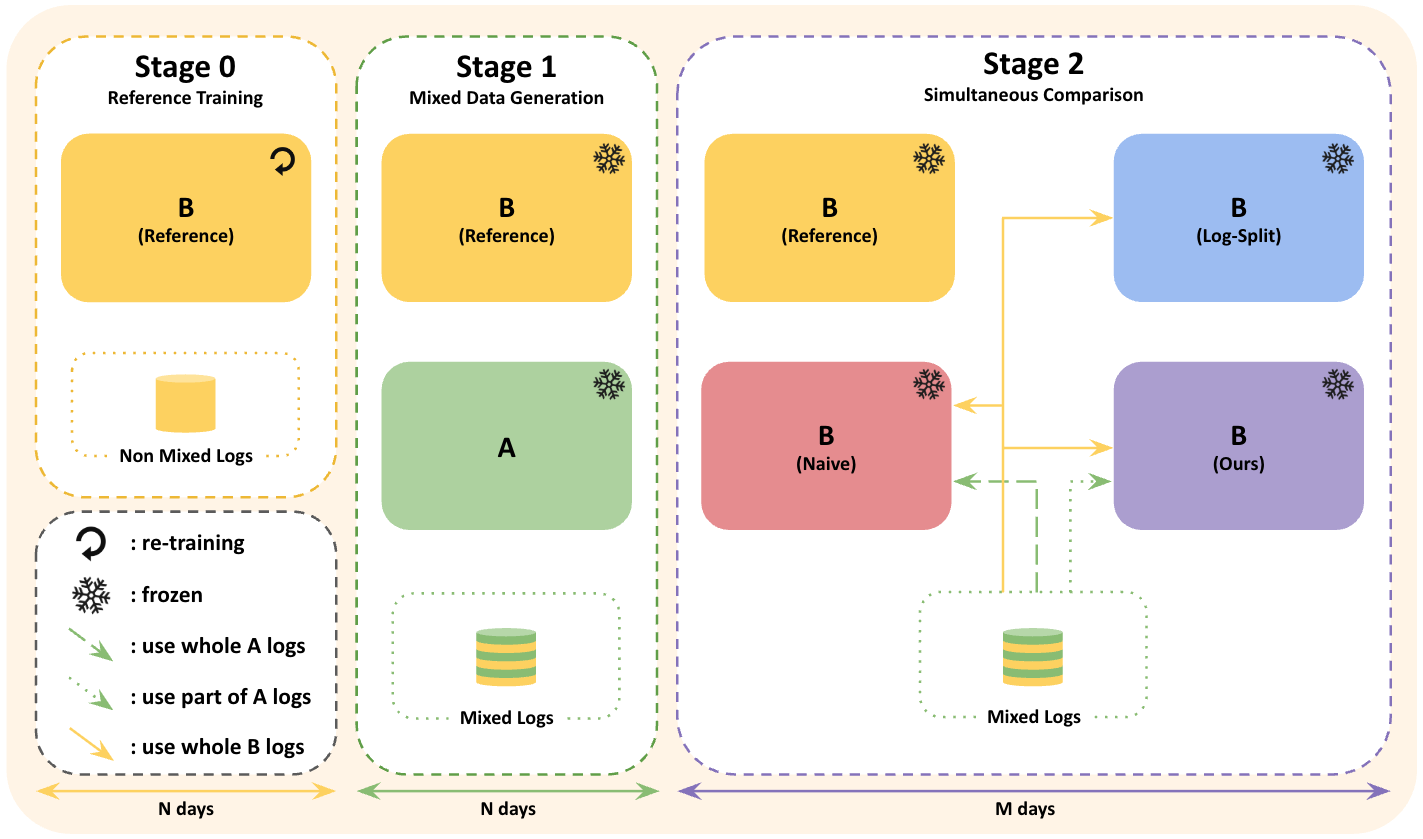}
              \caption{Three-stage online experiment protocol. Stage~0 ($N$ days)
              creates a pure reference model; Stage~1 ($N$ days) generates mixed
              serving logs via a 50/50 A/B split; Stage~2 ($M$ days) compares five
              $(k,l)$ operating points simultaneously at 20\% traffic each. This
              work uses $N = 2$, $M = 3$.}
              \Description{A three-stage experiment timeline. Stage 0 lasts N=2 days and
              trains a pure reference model. Stage 1 lasts N=2 days and runs a 50/50 A/B
              split to generate mixed serving logs. Stage 2 lasts M=3 days and evaluates
              five (k,l) operating points in parallel, each on 20 percent of traffic.}
              \label{fig:stages}
            \end{figure*}

\subsection{Production Setting}
\label{sec:online-prod-setting}

We conduct the experiment on a mobile advertising DSP at a
single SSP, covering one advertiser on that SSP. The SSP sends approximately 1.25M bid requests per
day. Our DSP platform uses a shared-parameter CTR/CVR prediction model~\cite{ma2018entirespacemultitaskmodel},
implemented as an multi-task learning(MTL)-based deep neural network.

The single advertiser--SSP scope is a deliberate choice to
\emph{bound revenue loss}. Our three-stage protocol (Section~\ref{sec:online-protocol}) fixes the traffic share of the interference-free reference model $B_{\text{ref}}$ at 100\%, 50\%, and 20\% across Stages~0--2. These shares are decided by what we need to \emph{measure}, not by which model performs best on live business KPIs. The experiment therefore cannot reallocate traffic toward better-performing models during the run and is expected to earn less revenue than a KPI-based rollout would. Confining it to a partial marketplace keeps the size of this revenue loss small.

\subsection{Three-Stage Protocol}
\label{sec:online-protocol}

Prior work~\cite{brennan2025symbiosis, lin2024} has demonstrated
the \emph{existence} of training data interference in shared-dataset
A/B tests. However, no existing protocol directly measures the bias
of a specific data-sharing strategy in production online. Our
three-stage protocol---Stage~0 reference model construction, Stage~1
mixed-log generation, Stage~2 simultaneous comparison at 20\%
traffic---fills this gap and is itself a methodological contribution.

We parameterize the length of Stages~0 and 1 as $N$ days and the length of Stage~2 as $M$ days. Here $N$ controls the reference model's training-data period, and $M$ controls the number of Stage~2 observation days and therefore the precision of the $\text{TE}_s(m)$ estimates defined in Section~\ref{sec:reference-preservation}. \emph{Longer is strictly better} on both axes. In this
work, we use $N = 2$ and $M = 3$, shortened from values typical of the
underlying production platform in order to \emph{minimize business
risk}: Stage~0 allocates 100\% of traffic to $B_{\text{ref}}$ and
Stage~2 splits traffic five ways at fixed 20\% shares, and both
allocations are mandated by the protocol rather than gated on observed
performance, so we cap $N$ and $M$ at the smallest values still
sufficient to establish an interference-free reference model and a
well-powered Stage~2 comparison.

\paragraph{Stage 0: Reference Model Training ($N$ days).}
$B_{\text{ref}}$ is deployed at 100\% traffic with
retraining enabled. After $N$ days, $B_{\text{ref}}$'s training dataset
is \emph{flushed} of logs produced by other models, establishing a
\emph{pure} ground truth---the production analogue of the Multiverse-$B$
reference model in our simulation.

\paragraph{Stage 1: Mixed Data Generation ($N$ days).}
The baseline $A$ and $B_{\text{ref}}$ are deployed at a 50/50 split,
both \emph{frozen}. This reproduces the standard A/B test scenario in
which serving logs from both models are interleaved, generating the
mixed training dataset $\mathcal{D}_A \cup \mathcal{D}_B$.

\paragraph{Stage 2: Simultaneous Comparison ($M$ days).}
Five models are deployed at 20\% traffic each, all frozen:
\begin{itemize}[nosep]
  \item $B_{\text{ref}}$: Stage~0 checkpoint (reference model)
  \item $\mathcal{F}_{1,\,1}$ (\emph{Naive}): no filter
  \item $\mathcal{F}_{0,\,0}$ (\emph{Log-Split}): discard $\mathcal{D}_A$ entirely
  \item $\mathcal{F}_{1,\,\varepsilon}$ (\textbf{ours}): bid-price filter only
  \item $\mathcal{F}_{0.5,\,0.5}$ (\textbf{ours}): joint bid-and-ad filter at the $(0.5, 0.5)$ interior point
\end{itemize}
Figure~\ref{fig:stages} illustrates the protocol.

\subsection{Evaluation: Reference-Preservation}
\label{sec:reference-preservation}

A successful $(k, l)$ filter should \emph{match} the interference-free reference model $B_{\text{ref}}$, not outperform it. For any scalar metric $m$, we define the \emph{Treatment
Effect} (TE) of strategy $s$ on $m$ as
\begin{equation*}
  \label{eq:te}
  \text{TE}_s(m) \;=\; \bigl|\, m(B_s) \,-\, m(B_{\text{ref}}) \,\bigr|.
\end{equation*}
Higher CTR or lower CPC should \emph{not} be read as improvement; only
closeness to the reference model counts, since any deviation---in
either direction---moves downstream business decisions away from the
interference-free ground truth. Our primary instantiation uses CPC,
the DSP's core business KPI, evaluated at the post-auction aggregate
level; CTR is reported as a secondary metric.

\subsection{Results}
\label{sec:online-results}

\paragraph{CPC alignment (primary).}
Table~\ref{tab:online-cpc} reports the aggregate CPC and
$\text{TE}_s(\text{CPC})$ of each filter.

\begin{table}[t]
  \caption{Reference-preservation in CPC across $(k,l)$ filters,
  Stage~2 (20\% traffic each). Smaller
  $\text{TE}_s(\text{CPC})$ indicates closer preservation of the
  interference-free reference model. CPC is measured in local currency.
  95\% CI from cluster bootstrap (B\,=\,10{,}000 resamples).}
  \label{tab:online-cpc}
  \centering
  \small
  \begin{tabular}{lrlrr}
    \toprule
    Strategy & CPC & 95\% CI & $\text{TE}(\text{CPC})$ & rank \\
    \midrule
    Reference                                           & 19.58 & [13.33,\, 32.16]  & ---            & --- \\
    $\mathcal{F}_{1,\, 1}$ (Naive)                      &  9.73 & [5.95,\, 20.08]   & 9.85           & 4 \\
    $\mathcal{F}_{0,\, 0}$ (Log-Split) & 10.74 & [7.08,\, 16.59] & 8.84           & 3 \\
    $\mathcal{F}_{1,\, \varepsilon}$ \textbf{(ours)}    & 20.02 & [12.16,\, 34.61]  & \textbf{0.44}  & 1 \\
    $\mathcal{F}_{0.5,\, 0.5}$ \textbf{(ours)}          & 24.94 & [10.78,\, 131.49] & \textbf{5.37}  & 2 \\
    \bottomrule
  \end{tabular}
\end{table}

$\mathcal{F}_{1,\,\varepsilon}$ and $\mathcal{F}_{0.5,\,0.5}$ occupy
ranks~1 and~2 on $\text{TE}_s(\text{CPC})$ in
Table~\ref{tab:online-cpc}: both filters land within
$1.0$--$1.3\times$ of the reference model's CPC, whereas
$\mathcal{F}_{0,0}$ (Log-Split) and $\mathcal{F}_{1,\,1}$ (Naive)
collapse to roughly $0.5\times$. The confidence intervals---particularly for $\mathcal{F}_{0.5,0.5}$---are wide given the limited per-filter Stage 2 sample size; see Section~\ref{sec:limitations} for further discussion.

\paragraph{CTR preservation (secondary).}
Table~\ref{tab:online-ctr} shows the same directional separation on the CTR axis: $\mathcal{F}_{0.5,\,0.5}$ and $\mathcal{F}_{1,\varepsilon}$ stay within $0.4$--$0.8$\%p of the reference model's CTR, whereas $\mathcal{F}_{0,\,0}$ (Log-Split) and $\mathcal{F}_{1,\,1}$ (Naive) diverge by $1.4$--$1.9$\%p. By the TE framing in Section~\ref{sec:reference-preservation}, the baselines' inflated CTR---mirroring their depressed CPC---is a deviation, not a gain.
    
\begin{table}[t]
  \caption{Reference-preservation in CTR, Stage~2.
  95\% CI from cluster bootstrap (B\,=\,10{,}000 resamples).}
  \label{tab:online-ctr}
  \centering
  \small
  \begin{tabular}{lrlrr}
    \toprule
    Strategy & CTR (\%) & 95\% CI & $\text{TE}(\text{CTR})$ \%p & rank \\
    \midrule
    Reference                                           & 1.19 & [0.71,\, 1.78] & ---           & --- \\
    $\mathcal{F}_{1,\, 1}$ (Naive)                      & 3.13 & [1.48,\, 5.25] & 1.94          & 4 \\
    $\mathcal{F}_{0,\, 0}$ (Log-Split) & 2.58 & [1.71,\, 3.57] & 1.38       & 3 \\
    $\mathcal{F}_{1,\, \varepsilon}$ \textbf{(ours)}    & 1.97 & [1.20,\, 2.98] & \textbf{0.78} & 2 \\
    $\mathcal{F}_{0.5,\, 0.5}$ \textbf{(ours)}          & 0.79 & [0.12,\, 1.77] & \textbf{0.40} & 1 \\
    \bottomrule
  \end{tabular}
\end{table}

\section{Limitations and Future Work}
\label{sec:limitations}

\paragraph{Theoretical surrogate.}
The theoretical analysis in Section~\ref{subsec:theory} is carried out on a ridge-regression surrogate with population-level quantities. The closed-form decomposition (Theorem~\ref{thm:main}) does not extend directly to the non-convex, finite-sample regime of MLP-based CTR models used in both the simulation and production deployment. The bound identifies which structural factors govern the filtered estimator's deviation, but the quantitative tightness of each factor under deep models remains an open question.

\paragraph{Online experiment scope.}
The present experiment covers a single advertiser--SSP pair over $N = 2$, $M = 3$ days (on production scale, larger value of $N$ is also feasible), so external validity rests on future replication across advertisers, SSPs and advertiser goals. The Stage~0 reference-model checkpoint is frozen for the full duration of Stages~1--2, so the reference does not adapt to market drift within the run. Of the $(k,l)$ filter grid, we deploy only one interior operating point $(0.5, 0.5)$ together with the boundary filters $\mathcal{F}_{1,\,\varepsilon}$ and $\mathcal{F}_{0,\,0}$ (Log-Split) and the reference model; the remaining interior is not tested.

A direct consequence of this narrow deployment is that the absolute magnitude of the CPC and CTR gaps between filters is expected to be SSP-specific due to the instability of each variants. The single-SSP restriction, adopted to minimize business risk, places the experiment in a relatively uncongested marketplace slice whose limited competitor pool may amplify per-filter differences relative to larger, more saturated SSPs. Accordingly, our preservation claim concerns the TE-based directional ordering rather than the absolute gap size, and the magnitudes reported in Tables~\ref{tab:online-cpc}--\ref{tab:online-ctr} should be read together with this SSP-specific context. Within-run statistical precision is also limited. Still, both $\mathcal{F}_{1,\,\varepsilon}$ and $\mathcal{F}_{0.5,\,0.5}$ achieve lower TE than $\mathcal{F}_{0,\,0}$ (Log-Split) and $\mathcal{F}_{1,\,1}$ (Naive) on both CPC and CTR.
    
\paragraph{Reference model without dedicated serving.}
Our protocol measures training-data interference against a reference model constructed via dedicated 100\% serving in Stage~0. An open question for future work is whether such a reference can be constructed without a dedicated serving stage, which would make the measurement protocol applicable to settings where dedicating full traffic to the experiment model is infeasible.

\section{Conclusion}
\label{sec:conclusion}

Shared-log training in RTB A/B tests biases each model's training data through two channels: ad-selection disagreement and bid-price disagreement. We formalized the Bid-Aware Filter Family (BAFF), a $(k,l)$-parameterized class of hard filters that interpolates between log-sharing and log-splitting by controlling tolerance to each channel independently. The family provides a structured search space: rather than committing to log-sharing or log-splitting, practitioners can locate the operating point that best preserves the unbiased reference in their specific deployment.

We validated the framework on two axes. In simulation, a $3 \times 3$ $(k,l)$ sweep revealed operating points with lower policy  distortion than either log-sharing or log-splitting alone. In a live RTB deployment on a commercial DSP, a three-stage measurement protocol confirmed that filter-based variants preserve the oracle's CPC and CTR more closely than both baseline. Both experiments demonstrate that the $(k,l)$ grid surfaces operating points that neither obvious strategy can reach, and that the best operating point differs between the two settings---underscoring the value of the search space itself. We recommend that practitioners sweep the $(k, l)$ grid on their own deployment to locate the best operating point for their setting.

\bibliographystyle{ACM-Reference-Format}
\bibliography{references}

\balance
\appendix
\section{Proofs}
\label{app:proofs}

\subsection{Proof of Lemma~\ref{lem:lb}}
\label{app:proof-lb}

\begin{lemma*}[Restatement of Lemma~\ref{lem:lb}]
Under Condition~\ref{cond:dom},
\[
\|\delta^*_N\| \;\geq\; (\alpha_R \beta - \alpha_F)\, \|\delta^*_F\|.
\]
\end{lemma*}

\begin{proof}
We have
\begin{align*}
\|\delta^*_N\|
&= \|\alpha_F \delta^*_F + \alpha_R \delta^*_R\| \\
&\geq \alpha_R \|\delta^*_R\| - \alpha_F \|\delta^*_F\| \\
&\geq (\alpha_R \beta - \alpha_F)\,\|\delta^*_F\|.
\end{align*}
The equality directly follows from Lemma~\ref{lem:decomp}. The first inequality follows from the reverse triangle inequality. The second inequality follows from Condition~\ref{cond:dom}, with $\beta > \alpha_F/\alpha_R$ ensuring the coefficient is positive.
\end{proof}

\subsection{Proof of Theorem~\ref{thm:main}}
\label{app:proof-main}

\begin{theorem*}[Restatement of Theorem~\ref{thm:main}]
Under Condition~\ref{cond:dom},
\[
\|\theta^*_{B+F} - \theta^*_B\|
\;\leq\;
\underbrace{\frac{\alpha_F}{\alpha_B + \alpha_F}}_{A}
\cdot
\underbrace{\frac{1}{\alpha_R \beta - \alpha_F}}_{B}
\cdot
\underbrace{\left(1 + \frac{\|\Sigma_N\|}{\lambda}\right)}_{C}
\cdot
\|\theta^*_N - \theta^*_B\|.
\]
\end{theorem*}

\begin{proof}
By Lemma~\ref{lem:core} applied to $P = B{+}F$ and Lemma~\ref{lem:decomp}, we can show
\begin{align*}
\theta^*_{B+F} - \theta^*_B
&= (\Sigma_{B+F} + \lambda I)^{-1}\,\delta^*_{B+F} \\
&= \tfrac{\alpha_F}{\alpha_B + \alpha_F}\,(\Sigma_{B+F} + \lambda I)^{-1}\,\delta^*_F.
\end{align*}
Taking Euclidean norms on both sides, we have
\begin{align*}
\|\theta^*_{B+F} - \theta^*_B\|
&\leq \tfrac{\alpha_F}{\alpha_B + \alpha_F}\,\|(\Sigma_{B+F}+\lambda I)^{-1}\|\,\|\delta^*_F\| \\
&\leq \tfrac{\alpha_F}{\alpha_B + \alpha_F}\cdot \tfrac{1}{\lambda}\cdot \|\delta^*_F\| \\
&\leq \tfrac{\alpha_F}{\alpha_B + \alpha_F}\cdot \tfrac{1}{\lambda}\cdot \tfrac{\|\delta^*_N\|}{\alpha_R\beta - \alpha_F} \\
&\leq \tfrac{\alpha_F}{\alpha_B + \alpha_F}\cdot \tfrac{1}{\lambda}\cdot \tfrac{\|\Sigma_N\| + \lambda}{\alpha_R\beta - \alpha_F}\,\|\theta^*_N - \theta^*_B\| \\
&= A \cdot B \cdot C \cdot \|\theta^*_N - \theta^*_B\|,
\end{align*}
where $A$, $B$, $C$ are the factors shown in the theorem statement.
The first inequality follows from submultiplicativity of the spectral norm. The second inequality follows from the fact that $\|(\Sigma_{B+F}+\lambda I)^{-1}\| \leq 1/\lambda$ and $\Sigma_{B+F} \succeq 0$. The third inequality follows from Lemma~\ref{lem:lb}, with $\beta > \alpha_F/\alpha_R$ ensuring positivity of the denominator. The fourth inequality follows from applying Lemma~\ref{lem:core} to $P = N$, which gives $\delta^*_N = (\Sigma_N + \lambda I)(\theta^*_N - \theta^*_B)$ and hence $\|\delta^*_N\| \leq (\|\Sigma_N\| + \lambda)\,\|\theta^*_N - \theta^*_B\|$.
\end{proof}

\section{Implementation Details}
\label{app:impl}

\paragraph{Architecture and training.} Both $A$ and $B$ are 2-hidden-layer MLPs with widths $(1024, 512)$ and ReLU activations, sharing one parameter set across the $K = 30$ ads (ad identity enters via one-hot concatenation). We use scikit-learn's \texttt{MLPClassifier}~\cite{sklearn_api} with Adam ~\cite{kingma2017adammethodstochasticoptimization}, learning rate $10^{-3}$, batch size $256$, early stopping on a $10\%$ validation split, and up to $500$ epochs.

\paragraph{Simulation scale.} Phase~0 draws 100{,}000 warm-up users. Phase~1 runs for 30 days at 100 users/day, yielding roughly 3{,}000 impressions per universe after auction filtering. Phase~2 evaluation uses 2{,}000 held-out users. The bid multiplier is $1.0$ (no bid shading). $A$ consumes $x_{1:20}$ and $B$ consumes $x_{1:50}$ of the 100-dimensional context.

\paragraph{Seeds and statistics.} Results aggregate $10$ independent seeds. Each seed controls the user draws, warm-up log noise, per-impression market-price noise, click realizations, the Phase-1 traffic split, and the held-out evaluation users; the ad catalog ($\boldsymbol{\phi}_a, \beta_a, v_a$) and MLP weight initialization are held fixed across seeds, so the reported variance reflects sampling stochasticity alone. Significance is reported via paired one-sided $t$-tests (seed-matched; the alternative hypothesis is that the candidate achieves lower KL than the baseline).

\end{document}